\documentclass[10pt,twocolumn,letterpaper]{article}

\usepackage{cvpr}
\usepackage{times}

\usepackage{graphicx}
\usepackage{amsmath}
\usepackage{amssymb}
\usepackage{epsfig}
\usepackage{epstopdf}
\usepackage{amsthm}
\usepackage{algorithm2e}
\usepackage{bm}
\usepackage{multirow}
\usepackage{overpic}

\makeatletter
\newcommand*{\algotitle}[2]{%
  \stepcounter{algocf}%
  \hypertarget{algocf.title.\theHalgocf}{}%
  \NR@gettitle{#1}%
  \label{#2}%
  \addtocounter{algocf}{-1}%
}
\makeatother

\newsavebox{\theorembox}
\newsavebox{\lemmabox}
\newsavebox{\claimbox}
\newsavebox{\corollarybox}
\newsavebox{\propositionbox}
\newsavebox{\examplebox}
\newsavebox{\conjecturebox}
\newsavebox{\algbox}
\newsavebox{\qbox}
\newsavebox{\problembox}
\newsavebox{\definitionbox}
\newsavebox{\assumptionbox}
\newsavebox{\hypothesisbox}
\savebox{\theorembox}{\noindent\bf Theorem}
\savebox{\lemmabox}{\noindent\bf Lemma}
\savebox{\claimbox}{\noindent\bf Claim}
\savebox{\corollarybox}{\noindent\bf Corollary}
\savebox{\propositionbox}{\noindent\bf Proposition}
\savebox{\examplebox}{\noindent\bf Example}
\savebox{\conjecturebox}{\noindent\bf Conjecture}
\savebox{\qbox}{\noindent\bf Question}
\savebox{\definitionbox}{\noindent\bf Definition}
\savebox{\problembox}{\noindent\bf Problem}
\savebox{\assumptionbox}{\noindent\bf Assumption}
\savebox{\hypothesisbox}{\noindent\bf Hypothesis}

\newtheorem{claim}{\usebox{\claimbox}}

\newtheorem{definition}{\usebox{\definitionbox}}

\newcommand{\argmax}{\operatornamewithlimits{argmax}}

\usepackage[pagebackref=true,breaklinks=true,letterpaper=true,colorlinks,bookmarks=false]{hyperref}

\cvprfinalcopy % *** Uncomment this line for the final submission

\def\cvprPaperID{3289} % *** Enter the CVPR Paper ID here

\ifcvprfinal\fi
\begin{document}
%%%%%%%%% TITLE
\title{OATM: Occlusion Aware Template Matching by Consensus Set Maximization
%\\ {\footnotesize or}\\ A Robust Approach to Occlusion Handling in Template Matching
}

\author{Simon Korman$^{1,3}$\\
$^{1}$Weizmann Institute of Science\\
%{\tt\footnotesize simon.korman@gmail.com}
\and
Mark Milam$^{2}$\\
$^{2}$Northrop Grumman\\
\and
Stefano Soatto$^{1,3}$\\
$^{3}$UCLA Vision Lab\\
}

\maketitle
%\thispagestyle{empty}
%%%%%%%%% ABSTRACT
\begin{abstract}
We present a novel approach to template matching that is efficient, can handle partial occlusions, and comes with provable performance guarantees. A key component of the method is a reduction that transforms the problem of searching a nearest neighbor among $N$ high-dimensional vectors, to searching neighbors among two sets of order $\sqrt{N}$ vectors, which can be found efficiently using range search techniques. This allows for a quadratic improvement in search complexity, and makes the method scalable in handling large search spaces. The second contribution is a hashing scheme based on consensus set maximization, which allows us to handle occlusions. The resulting scheme can be seen as a randomized hypothesize-and-test algorithm, which is equipped with guarantees regarding the number of iterations required for obtaining an optimal solution with high probability. The predicted matching rates are validated empirically and the algorithm shows a significant improvement over the state-of-the-art in both speed and robustness to occlusions.
\end{abstract}

%%%%%%%%% BODY TEXT
\vspace{-10pt}
%========================================================================
\section{Introduction}
%========================================================================

Matching a template $T$ (a small image) to a target $I$ (a larger image) can be trivial to impossible depending on the relation between the two. In the classical setup, when $I$ is a digital image and $T$ is a subset of it, this amounts to a search over the set of $N$ discrete 2D-translations, where $N$ would be the number of pixels in $I$%(minus a border of width proportional to the dimensions of $T$)
. When $T$ and $I$ are images of the same scene taken from different vantage points, their relation can be described by a complex deformation of their domain, depending on the shape of the underlying scene, and of their range, depending on its reflectance and illumination. For a sufficiently small template, such deformations can be approximated by an {\em affine} transformation of the domain (``warping''), and an affine (``contrast'') transformation of the range \ldots {\em except for occlusions:} An arbitrarily large portion of the template, including all of it, may be occluded and therefore have no correspondent in the target image. 

This poses a fundamental problem to many low-level tasks: To establish local correspondence (co-visibility), the template should be large, so as to be discriminative. But increasing the area increases the probability that its correspondent in the target image will be occluded, which causes the correspondence to fail, {\em unless occlusion phenomena are explicitly taken into account.}

In this work  we model occlusions explicitly as part of a robust template matching process where the co-visible region is assumed to undergo affine deformations of the domain and range, up to additive noise. We search for transformations that maximize {\em consensus,}  that is the size of the co-visible set, in a manner that is efficient and comes with provable convergence guarantees.

% % % % % Figure Instances of Occlusion Experiment
\begin{figure}\vspace{-2pt}
\begin{center}
\includegraphics[width = 1\linewidth]{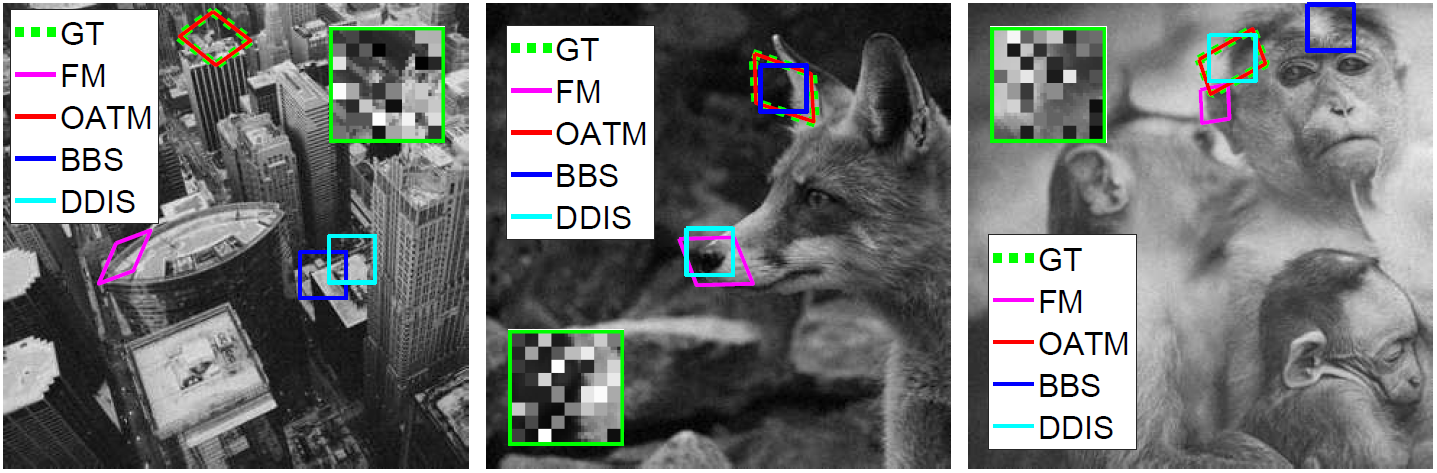}

\caption{$\hspace{-0.2pt}$\textbf{Instances of the occlusion experiment (Sec.~\ref{sec:occlusion.experiment})} A template (overlaid in green) that is 60\% occluded by random blocks is searched for in an image. {\footnotesize{\textbf{OATM}}} shows the best results in dealing with significant deformation and occlusion (use zoom for detail).}\label{fig.occlusion.types}
\end{center}\vspace{-17pt}
\end{figure}
% % % % % Figure Instances of Occlusion Experiment

Efficiency comes from the {first contribution} - a {\em reduction method} whereby the linear search of nearest neighbors for the $d$-dimensional template $T$ through $N$ versions in the target image is converted to a search among two {\em sets} of vectors, {\em with each set of size $O(\sqrt{N})$} (Sect. \ref{sec:reduction}). This reduces the search complexity from $O(N)$ to $O(\sqrt{N})$, which is practical even for very large search spaces, such as the discretized space of {\em affine transformations.}

For this method to work, we need a {\em hashing scheme that is compatible with  occlusions}, which we achieve by adapting the scheme of Aiger {\em et al.} \cite{aiger2013random}, leading to our {second contribution:} Rather than reporting close neighbors under the Euclidean $\ell_2$ norm, we are interested in reporting pairs of vectors that are compatible, up to a threshold, on a maximum (co-visibility) consensus set. Our hashing scheme is akin to a random consensus (RANSAC-type) procedure under the $\ell_\infty$ norm (Sect. \ref{sec:search}).

Finally, our {third contribution} is an analysis of the algorithm (Sect. \ref{sec:analysis}), specifically regarding guarantees on the number of candidate hypotheses required for obtaining the optimal solution, in the sense of maximal inlier rate, within a certain probability.

While for many low-level vision tasks \emph{speed}, not convergence guarantee, is the key, there are applications where being able to issue a certificate is important, such as high-assurance visual pose estimation for satellite maneuvering. In our case, we achieve both speed and assurance, all the while being able to handle occlusions, which allows using larger, and therefore more discriminative, templates.

The algorithm is rather generic and is presented for a general geometric transformation of the domain space, while possible explicit decompositions are given for the 2D-translation and 2D-affine groups.
In the experimental section, our algorithm is shown empirically to outperform the state-of-the-art in affine template matching \cite{korman2013fast} both in terms of efficiency and robustness to occlusion. In addition, it shows some clear advantages over some modern image descriptors on the recent HPatches \cite{hpatches_2017_cvpr} benchmark.

\subsection{Related work}

% traditional PM
Research in template matching algorithms has focused heavily on efficiency, a natural requirement from a low level component in vision systems. This was largely achieved in the limited scope of 2D-translation and $\ell_p$ similarity, where full-search-equivalent algorithms accelerate naive full-search schemes by orders of magnitude \cite{ouyang2012performance}.
%The need for efficiency is taken to an extreme in
unlike in real-time applications, such as robotic navigation and augmented reality, there are applications where accuracy and performance guarantees are important, such as high-assurance pose estimation for high-value assets, such as satellites or industrial equipment. This requires extending the scope of research in several aspects.

% geometric PM
One line of works focuses on \emph{geometric deformations} due to camera or object motion.
Early works such as \cite{Fredriksson-PhD,tsai2002rotation} extend the sliding window approaches to handle rotation and scale.
The Fast-Match algorithm~\cite{korman2013fast} was designed to handle 2D-Affine transformations. It minimizes the sum-of-absolute-differences using branch-and-bound, providing probabilistic global guarantees. \cite{zhang2015fast} uses a genetic algorithm to sample the 2D-affine space.

% photometric PM
To achieve \emph{photometric invariance}, \cite{hel2014matching} introduced a fast scheme for matching under non-linear tone mappings, while \cite{elboer2013generalized} used the Generalized Laplacian distance, which can handle multi-modal matching. Our method can provide affine photometric invariance, i.e., up to global brightness and contrast changes.

%1
In this work we propose a \emph{quadratic} improvement upon the runtime complexity of these methods, which depends linearly on the size of the search-space ({\em i.e.,} exponential in its dimension).
More recently we are seeing attempts at matching under 2D-homographies using deep neural networks~\cite{detone2016deep,nguyen2017unsupervised}, although these methods do not provide any guarantees and like the previously mentioned methods - they were not designed to handle \emph{partial occlusion}.

%2
Two recent works can handle \emph{both} geometric deformations and partial occlusion through similarity measures between rectangular patches: the \noindent Best Buddies Similarity (BBS) measure \cite{dekel2015best},  based on maximizing the number of mutual nearest-neighbor pixel pairs, and Deformable Diversity Similarity (DDIS) \cite{talmi2016template}, that examines the nearest neighbor field between the patches. DDIS dramatically improves the heavy runtime complexity of BBS, but is limited in the extent of deformation it can handle, since it penalizes large deformations. Also, the sliding window nature of these methods limits the extent of occlusion they can handle. While OATM is limited to handling \emph{rigid} transformations, it is provably able to efficiently handle high levels of deformation and occlusion.

%3
Another relevant and very active area of research is learning discriminative descriptors for image patches (natural patches or those extracted by feature detectors), from the earlier {SIFT}~\cite{lowe1999object} and variants \cite{calonder2010brief,rublee2011orb} to the more recent \cite{simo2015discriminative,balntas2016learning,han2015matchnet}. We show OATM to be superior in its ability to match under significant deformation and occlusion.

%4
Lastly, the problem of occlusion handling was addressed in many other areas of computer vision, including
tracking \cite{zhang2015robust,zhang2014partial,joshi2007synthetic}, segmentation \cite{yang2015self},
image matching \cite{yang2015coarse}, multi-target detection \cite{Baque_2017_ICCV}, flow \cite{Hur_2017_ICCV} and recognition \cite{Osherov_2017_ICCV}.

Within a landscape of ``X-with-deep-learning'' research, our work is counter-tendence: We find that the need to provide provable guarantees in matching, albeit relevant to niche applications, is underserved, and data-driven machine learning tools are not ideally suited to this task.

%========================================================================
\section{Method}
%========================================================================

\subsection{Problem Definition}

In template matching, one assumes that a template $T$ and an image $I$ are related by a geometric transformation of the domain  $F=\{f:\mathbb{R}^2\rightarrow\mathbb{R}^2\}$ and a photometric transformation of the range space. The goal is to determine the transformation of the domain, despite transformations of the range. 
Here we assume that both $T$ and $I$ are discretized, real valued, square images, and hence can be written as $T:\{1,\ldots,n\}^2\rightarrow\mathbb{R}$ (and similarly $I:\{1,\ldots,m\}^2\rightarrow\mathbb{R}$), where $T$ and $I$ are $n\times n$ and $m\times m$ images, respectively.
The set of transformations $F$ can be approximated by a discrete set of size $N$, possibly large, up to a desired tolerance. For example, in the standard 2D-translation setup, the set $F$ contains all possible placements of the template over the image at single pixel offsets, and hence $N=|F|\approx (m-n)^2$ with a tolerance of one pixel. Moreover, in our analysis we will assume nearest-neighbor interpolation (rounding) which allows us to simplify the discussion to fully discretized transformations of the form $f:\{1,\ldots,n\}^2\rightarrow\mathbb{Z}^2$.

With a slight abuse of notation we indicate with $p\in T$ (and likewise $q\in I$) a pixel $p$ in the template domain $\{1,\ldots,n\}^2$ and $T(p)$ will denote its real valued intensity.

For a given transformation $f$, we define the (photometric) residual, or reprojection error, at pixel $p\in T$ by $res_f(p) = |T(p)-I(f(p))|$. 
The known ``brightness constancy constraint" guarantees that the residual can be made small (to within a threshold) by at least one transformation $f$. However, it is only valid for portions of the scene that are Lambertian, seen under constant illumination and most importantly: co-visible (unoccluded).

We are now ready to pose Occlusion-Aware Template Matching (OATM) as a Consensus Set Maximization (CSM) problem, where we search for a transformation under which a maximal number of pixels are co-visible, {\em i.e.,} mapped with a residual that is within the a threshold.

\begin{definition}
\emph{[{Occlusion-Aware Template Matching (OATM)}]}
For a given error threshold $t$, find a transformation $f^*$ given by:\vspace{-7pt}
\begin{equation}\label{eq.OATM}
f^* = \argmax_{f\in F} \sum_{p\in T} [res_f(p)\le t]\vspace{-9pt}
\end{equation}
where $[\cdot]$ represents the indicator function.
\end{definition}

Our reduction to a product space relies extensively on a distance notion between geometric transformations (which depends on the source domain - the template $T$).
\begin{definition}\label{def.distance} \emph{[Distance $\Delta$ between transformations]}
Let $f_1,f_2\in F$. We define the distance $\Delta(f_1,f_2)=\max_{p\in T}\|f_1(p)-f_2(p)\|$
where $\|\cdot\|$ represents the Euclidean distance in the (target) domain of the image $I$.
\end{definition}

\subsection{Reduction to a Product Space}
\label{sec:reduction}

Recall (Equation \eqref{eq.OATM}) that our goal is to find an optimal transformation $f^*$, one whose residual \begin{equation}\label{eq.residual.orig}
res_{f^*}(p) = |T(p)-I(f^*(p))|
\end{equation}
is below a threshold $t$ at as many pixels $p\in T$ as possible. In order to optimize \eqref{eq.OATM} we would need to compare $T$ to $N$ possible target vectors $I(f(T))$ (all possible transformed templates in the target image).

The main idea here will be to enumerate the search space in a very different way. On the source image side we define a set $U$ of templates (vectors) obtained by local perturbations of the template $T$, while on the target side we define a set $V$ of templates that ``covers" the target image $I$ in a sense that every target template location will be close to one of those in $V$. In such a way, if a copy of the template appears in the image, there must be a pair of similar templates (vectors) $u\in U$ and $v\in V$. Refer to  Figure \ref{fig.illust.decomp} to get the intuition for the 2D-translation case.

Formally, for a given tolerance $\epsilon > 0$, let $f\in F$ be a transformation such that $\Delta(f,f^*)<\epsilon$.
For an arbitrary $p'\in T$, if we assume the existence of some $p\in T$ such that $f(p)=f^*(p')$, which is the case in our model under the assumption of co-visibility,
by substituting $p'={f^*}^{-1}(f(p))$ in Equation~\eqref{eq.residual.orig}, we get:

\begin{equation}\label{eq.residual.2}
res_{f^*}(p') = |T({f^*}^{-1}(f(p)))-I(f(p))|\;.
\end{equation}

If we set $h={f^*}^{-1}\circ f$, we can write:
\begin{equation}\label{eq.residual.3}
  res_{f^*}(p') = |T(h(p))-I(f(p))|
\end{equation}
for pixels $p$ in the sub-template $T_h=\{p\in T\: : \; h(p)\in T\}$, for which $h(p)=p'\in T$.

Regarding $h$, since we know that $\Delta(f,f^*)<\epsilon$, it is easy to see that $\Delta(h,id)<\epsilon/s(f^*)$, where $id$ is the identity transformation and $s(f^*)$ is the minimal scale of $f^*$, defined by $s(f)=\min_{p\in T}\|f(p)\|/\|p\|$.

If we call $\epsilon' = \epsilon/s(f^*)$ we can now define the restricted subset of functions (which is a \emph{ball} of radius $\epsilon'$ around the identity, in the function space $F$):
\begin{equation}
F_{\epsilon'}=\{h\in F\; : \; \Delta(h,id)<\epsilon'\}
\end{equation}
Let $Net_\epsilon(F)$ be an arbitrary $\epsilon$-net over the space $F$, with respect to the distance $\Delta$. Namely, for any $f\in F$ there exists some $f'\in Net_\epsilon(F)$ such that $\Delta(f,f')<\epsilon$.

The result is that we have decomposed the search for an optimal $f^*\in F$ in Eq.~\eqref{eq.OATM}, to the search of the equivalent (recall that $h={f^*}^{-1}\circ f$) optimal pair $(h,f)$ in the product space $F_{\epsilon'} \times Net_\epsilon(F)$. Namely, we can reformulate the OATM problem (Equation \eqref{eq.OATM}) as:
\begin{equation}\label{eq.OATM2}
\hspace{-2pt}f^* = \;\argmax_{\substack{h\in F_{\epsilon'} \\ f\in Net_\epsilon(F)}} \;\sum_{p\in T_h} \frac{1}{|T_h|}\Big[|T(h(p))-I(f(p))|\le t\Big]
\end{equation}
For simplicity of description and implementation we can work with a fixed subtemplate $T'$ of $T$, defined by the intersection of all sub-templates $\{T_h\}_{h\in F_{\epsilon'}}\;$, which results in:
\begin{equation}\label{eq.OATM2}
f^* = \;\argmax_{\substack{h\in F_{\epsilon'} \\ f\in Net_\epsilon(F)}} \;\sum_{p\in T'} \Big[|T(h(p))-I(f(p))|\le t\Big]
\end{equation}
It may appear that, up to this point, we stand to gain nothing, since under any reasonable discretization of the transformation sets $Net_\epsilon(F)$ and $F_{\epsilon'}$, it holds that $|F|\approx |Net_\epsilon(F)| \cdot |F_{\epsilon'}|$, i.e. that the size of the search space remains unchaged. However, this decomposition allows us to design preprocessing schemes for two sets of vectors\footnote{$h(T')$ and $f(T')$ are shorthands for $\{h(p)\}_{p\in T'}$ and $\{f(p)\}_{p\in T'}$}
\begin{eqnarray}
  U &=& \{T(h(T'))\}_{h\in F_{\epsilon'}} \label{eq.setU} \\
  V &=& \{I(f(T'))\}_{f\in Net_\epsilon(F)} \label{eq.setV}
\end{eqnarray}
 in a manner that enables an efficient search over the terms $|T(h(p))-I(f(p))|$ from \eqref{eq.OATM2} for all $(h,f)\hspace{-2pt}\in\hspace{-2pt} F_{\epsilon'} \hspace{-2pt}\times \hspace{-2pt}Net_\epsilon(F)$. Efficiency comes from designing the product space in a way that the sets $U$ and $V$ have approximately equal size ($\sqrt{N}$) and from using a search algorithm whose complexity depends on the \emph{sum} of the space sizes (order $\sqrt{N}$), and not on their \emph{product} (of size $N$). We provide explicit decompositions for the 2D-translation and 2D-affine spaces.

%-------------------------------------------------------------------------
\subsection{Search by Random Grid based Hashing}
\label{sec:search}

We have transformed the problem of matching between a single vector and $N$ target vectors to that of finding matching vectors between two sets of $\thicksim\hspace{-4pt}\sqrt{N}$ vectors. Matching between a pair of high-dimensional point sets is a classical problem in the search literature, clearly related to the problem of finding all close neighbors in a single point-set. Our approach is based on random grid hashing \cite{aiger2013reporting} - an algorithm that is straightforward to implement and which has been shown to work well in practice \cite{aiger2013random}.

In \cite{aiger2013reporting}, to hash a collection of $d$ dimensional points, the space is divided into cells by laying a randomly shifted uniform grid (each cell is an axis-parallel cube with side-length $c$). The points are arranged accordingly in a hash table and then all pairs of points that share an entry in the hash table are inspected, reporting those whose distance is below the specified threshold. The process is then repeated a suitable number of times in order to guarantee, with high probability, that all or most pairs of close points are reported.

Unlike the work of Aiger \etal~\cite{aiger2013reporting,aiger2013random} that uses the $\ell_2$ norm to measure the similarity between vectors, we use the number of coordinates whose absolute difference is below a threshold. Furthermore, we replace the dimensionality reduction in \cite{aiger2013random} (a Johnson--Lindenstrauss transform) by a random choice of a small number of coordinates (pixels), in order to enable matching under occlusions. These changes require a different analysis of the algorithm.
Refer to Algorithm~\ref{alg:basic.hash} for a summary of our basic hashing module.

%%%%%%%%%%%%%%%%%%% algorithm 1 : Hashing %%%%%%%%%%%%%%%%%%%%%%%
\begin{algorithm}[t]\vspace{-0pt}
\algotitle{IRE}{alg.basic.hash}
\SetAlgoLined
\SetKwInput{Input}{input}\SetKwInput{Output}{output}

\hrulefill

\Input{Sets $U$ and $V$ of vectors in $\mathbb{R}^d$; threshold $t$;}

\Output{A vector pair $(u,v)\in U\times V$ with maximal found consensus set (inlier rate)}

\textbf{parameters:} Sample dimension $\hat{d}$; cell dimension $c$;\vspace{-5pt}

\hrulefill

\hspace{-6pt}\parbox{.91\columnwidth}{
\begin{enumerate}\setlength\itemsep{-2pt}
\item Pick $\hat{d}$ random dimensions out of $1,\ldots,d$.
\item Let $\hat{U}$ and $\hat{V}$ be the vector sets $U$ and $V$ reduced to the $\hat{d}$ random dimensions.
\item Generate a random $\hat{d}$-dimensional offset vector ${o}$ in $[0,c]^{\hat{d}}$.
\item Map each vector in $\hat{U}$ and $\hat{V}$ into a $\hat{d}$-dimensional integer, according to $Map(\hat{v})=\lfloor (\hat{v} + {o})/c \rfloor$.
\item Arrange the resulting integers into a hash table using any hash function from $\mathbb{N}^{\hat{d}}$ to $\{1,\ldots,|U|\}$.
\item Scan the hash table sequentially, where for each pair of vectors $\hat{u}$ and $\hat{v}$ that share a hash value, count the number of inlier coordinates in $i\in\{1,\ldots,d\}$ (those for which $|u(i)-v(i)|\le t$).
\item Return a pair ${u,v}$ with maximal found inlier rate
\end{enumerate}\vspace{-8pt}
}

\hrulefill

\medskip
\caption{\small Consensus Set Maximization in vector sets.\label{alg:basic.hash}} %
\end{algorithm}\vspace{-1pt}
%%%%%%%%%%%%%%%%%%% algorithm 1: Hashing %%%%%%%%%%%%%%%%%%%%%%%

%%%%%%%%%%%%%%%%%%%%%%%%%%%%%%%%%%% figure: decompositions of 2DT %%%%%%%%%%%%%%%%%%%
\begin{figure*}\vspace{-12pt}
\begin{center}
\addtolength{\tabcolsep}{21pt}
\begin{tabular}{ c  c}
    \hspace{-21pt}\includegraphics[width = 0.42\linewidth]{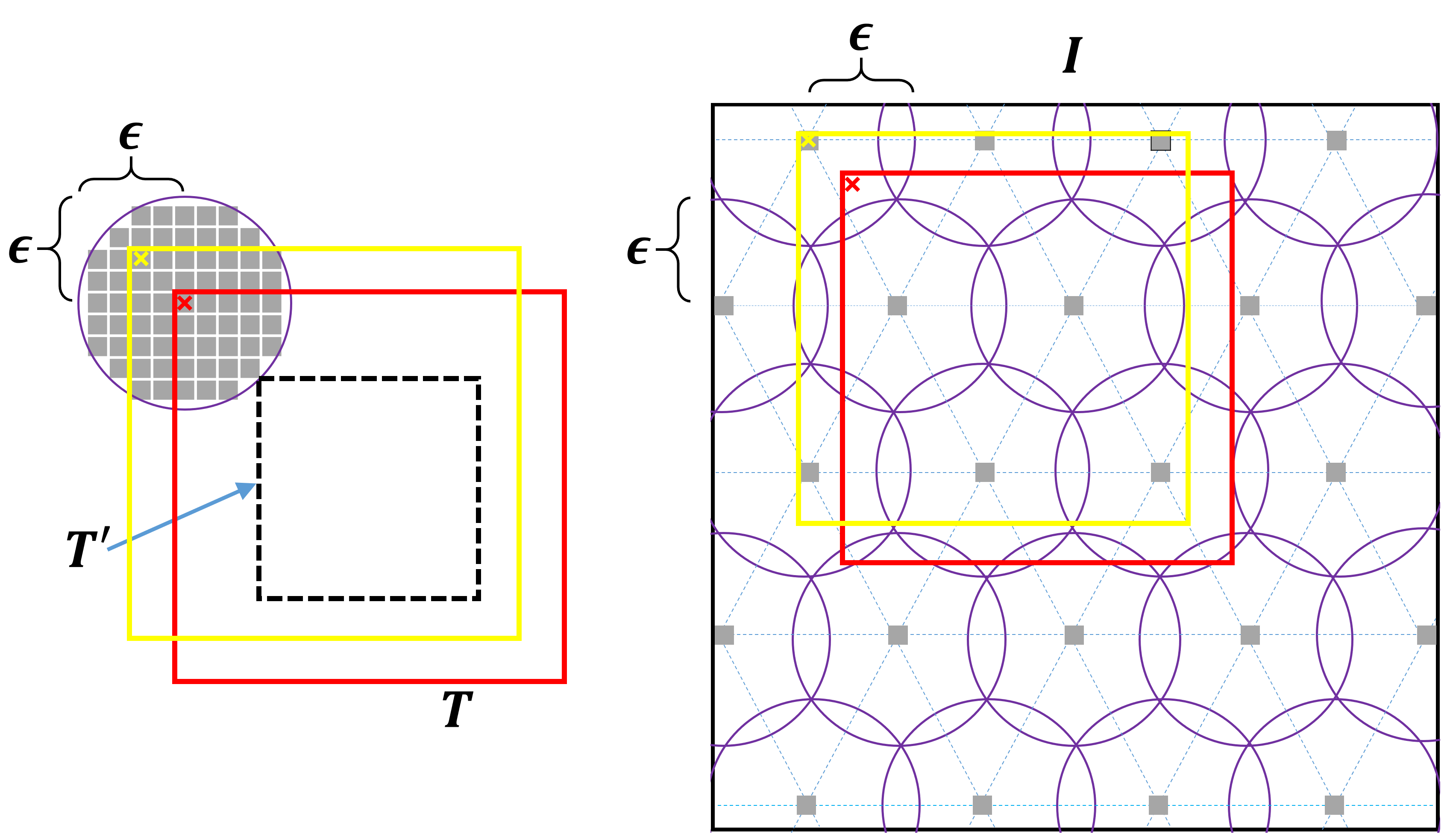} &
    \includegraphics[width = 0.42\linewidth]{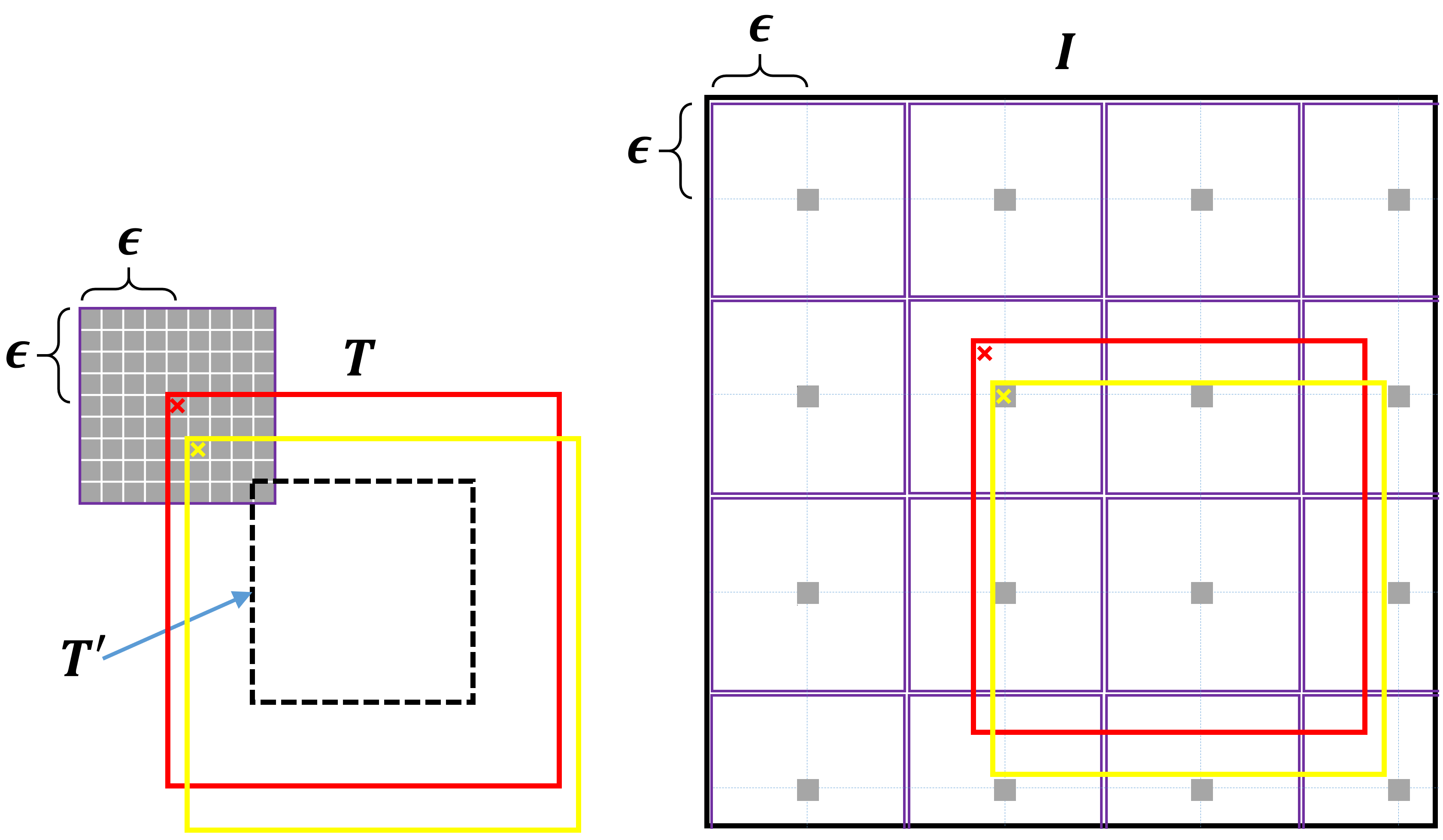} \\ \vspace{-12pt}\\
    (a) $\epsilon$-net construction & (b) simple construction\\ \vspace{-8pt}\\
\end{tabular}\vspace{-1pt}
\caption{\textbf{Illustration of two possible decompositions for 2D-translation}. In each of (a) and (b) the sets of sampled vectors (templates) $U$ (from $T$) and $V$ (from $I$) are represented by gray pixels which denote the top left corner position of the sampled templates. If the Template (red) appears in the target image, there will be a respective pair of matching samples in $U$ and $V$ (shown in yellow). The parameter $\epsilon$ is taken such that the number of samples (number of gray squares) on both sides is approximately equal (both approximately $\sqrt{N}$).}\label{fig.illust.decomp}
\end{center}\vspace{-18pt}
\end{figure*}
%%%%%%%%%%%%%%%%%%%%%%%%%%%%%%%%%%% figure: decompositions of 2DT %%%%%%%%%%%%%%%%%%%

%-------------------------------------------------------------------------
\subsection{Analysis}
\label{sec:analysis}

The main result needed for a high-assurance template matcher is a guarantee on the success probability of Algorithm~\ref{alg:basic.hash}. The following term will be used in our claims:
\begin{equation*}\label{eq.hypergeometric}
P(\alpha,d,\hat{d}) = \frac{\binom{\alpha d}{\hat{d}}}{\binom{d}{\hat{d}}}=\frac{\alpha d\cdot (\alpha d-1)\cdot\ldots\cdot(\alpha d-\hat{d}+1)}{ d\cdot ( d-1)\cdot\ldots\cdot( d-\hat{d}+1)}
\end{equation*}

\begin{claim} \label{Clm.main}
\textnormal{\text{\textbf{[analysis of Algorithm~\ref{alg:basic.hash}]}}}
Algorithm~\ref{alg:basic.hash} succeeds (reports a pair ${u,v}\in U\times V$ with maximal possible inlier rate of $\alpha$) with probability at least
\vspace{-4pt}\begin{equation}\label{eq.success.probab.simple.alg1}
P(\alpha,d,\hat{d})\cdot \left( 1-\frac{t}{c} \right) ^{\hat{d}}
\end{equation}
\end{claim}

\begin{proof}
The derivation is straightforward, since the algorithm succeeds if a pair of optimal matching vectors ${u,v}$ collide in the hash table. A collision is guaranteed to occur, given a combination of two events. First, the event that the set of the $\hat{d}$ sampled dimensions is a subset of the $\alpha d$ inlier dimensions. This occurs with probability $P(\alpha,d,\hat{d})$, since this is a hyper-geometric distribution with $\alpha d$ success items among a population of $d$, with $\hat{d}$ samples all required to be success items. Second, we need to multiply by the probability that a collision occurred subject to the randomness in the grid offset. In this case, the $\hat{d}$-dimensional $\hat{u}$ and $\hat{v}$ differ by at most $t$ in each coordinate. Therefore, and since the offset is uniform and independent between coordinates, $\hat{u}$ and $\hat{v}$ are mapped into the same cell (and hence collide in the hash table) with probability at least $(\frac{c-t}{c})^{\hat{d}}=(1-\frac{t}{c})^{\hat{d}}$.
\end{proof}

\begin{claim} \label{Clm.main.strong}
\textnormal{\text{\textbf{[analysis of Algorithm~\ref{alg:basic.hash} - \normalfont{stronger version}]}}}
Assume there exists a pair ${u,v}\in U\times V$ which are identical up to a zero-mean Gaussian noise with standard deviation $\sigma$ at an $\alpha$-fraction of their coordinates.  Algorithm~\ref{alg:basic.hash} succeeds (reports a pair ${u,v}\in U\times V$ with inlier rate at least $\alpha$) with probability at least
\begin{equation}\label{eq.success.probab.alg1}
 P(\alpha,d,\hat{d})\cdot \Big(\int_{0}^{c}(1-\frac{x}{c})\cdot \frac{\sqrt{2}}{\sigma\sqrt{\pi}}\cdot exp(-\frac{x^2}{2\sigma^2})dx\Big)^{\hat{d}}\:
\end{equation}
\end{claim}

\begin{proof}
The only difference here compared to the previous claim is regarding the probability of vectors of inlier coordinates falling into a single cell. The difference is in the definition of inliers, where here we not only assume a maximal absolute difference of $t$ at each coordinate but we rather make the stronger (but realistic) assumption that the vectors at inlier coordinates differ only due to Gaussian noise of a known standard deviation. In such a case, the absolute difference per coordinate follows a folded Gaussian distribution (see e.g. \cite{leone1961folded}), and therefore we integrate over the possible absolute differences $x$ in the range $[0,c]$. 
\end{proof}

%-------------------------------------------------------------------------
\subsection{Occlusion-Aware Template Matching}
%%%%%%%%%%%%%%%%%%% algorithm 2 : OATM %%%%%%%%%%%%%%%%%%%%%%%

\begin{algorithm}[b]\vspace{-8pt}
\algotitle{IRE}{alg.oatm}
\SetAlgoLined
\SetKwInput{Input}{input}\SetKwInput{Output}{output}

\hrulefill

\Input{template $T$ and image $I$; threshold $t$; $\qquad$family of transformations $F$ (of size $N$);} 

\Output{$f\in F$ with maximum consensus (Eq.~\eqref{eq.OATM})}

\hrulefill

\hspace{-6pt}\parbox{.91\columnwidth}{
\begin{enumerate}\setlength\itemsep{-2pt}
\item Decompose $F$ into the product $F_{\epsilon'} \times Net_\epsilon(F)$ choosing an $\;\epsilon\;$ s.t. $|F_{\epsilon'}|\approx |Net_\epsilon(F)|\approx \sqrt{N}$.\label{step.decompose}
\item Construct the vector sets $U$ and $V$ (Eqs.~\eqref{eq.setU}-\eqref{eq.setV}). 
\item\textbf{repeat} Algorithm~\ref{alg:basic.hash} for $k$ times (with $U$, $V$ and $t$) to obtain transformations $\{f_i\}_{i=1}^k$.
\item\textbf{return} the transformation $f_i$ with largest consensus set (Eq.~\eqref{eq.OATM}).
\vspace{-8pt}
\end{enumerate}
}

\hrulefill

\medskip

\caption{\small{\textbf{OATM}:} {{Occlusion Aware Template Matching}\vspace{-21pt}}\label{alg:oatm}}

\end{algorithm}\vspace{-0pt}
%%%%%%%%%%%%%%%%%%% algorithm 2: OATM %%%%%%%%%%%%%%%%%%%%%%%

Given Algorithm~\ref{alg:basic.hash} and its performance guarantees, we can now specify our complete OATM template matching algorithm. The template matcher will run Algorithm~\ref{alg:basic.hash} a certain number of times and return the target location in the image, which corresponds to the overall best pair of vectors found. As a reminder, Algorithm~\ref{alg:basic.hash} returns a pair of vectors which are of the form $\{T(h(p))\}_{p\in T'}$ and $\{I(f(p))\}_{p\in T'}$, which suggests the pair of transformations $(h,f)$ as a candidate solution, from which a single transformation $f^*=f\circ h^{-1}$ can be extracted.

There are two reasons to evaluate directly the inlier rate $P^* = \frac{1}{|T|}\sum_{p\in T}\big[ |T(p)-I(f^*(p))|\le t\big]$ instead of the proxy $\frac{1}{|T'|}\sum_{p\in T'}\big[|T(h(p))-I(f(p))|\le t\big]$. One is to avoid interpolation errors by applying the concatenated transformation $f^*=f\circ h^{-1}$ directly. The second and more important one is that the detected inlier rate reflects only pixels of $T'$ in a sub-template of $T$.

Occlusion-Aware Template Matching (OATM) is summarized in Algorithm~\ref{alg:oatm}. It consists of running Algorithm~\ref{alg:basic.hash} for $k$ iteration. If we denote by $P_\alpha$ the success probability of Algorithm~\ref{alg:basic.hash}, given in Equation~\eqref{eq.success.probab.alg1} of Claim~\ref{Clm.main.strong}, it holds that the success probability of Algorithm~\ref{alg:oatm} is at least: \vspace{-6pt}
\begin{equation}\label{eq.success.probab.alg2}
1 - (1-P_\alpha)^k\vspace{-4pt}
\end{equation}
and conversely, the number of iterations $k$ needed in order to succeed with a predefined probability $p_0$ (e.g. 0.99) is: $\log(1-p_0)/\log(1-P_\alpha)$.

It is important to note that the number of iterations $k$ can be determined \emph{adaptively}, based on the findings of previous rounds. As is common in the RANSAC pipeline, every time the best maximal consensus (inlier rate) is updated, the number of required iterations is decreased accordingly.

Notice that the algorithm is generic with respect to the underlying transformation space $F$. It does however require the knowledge of how to efficiently decompose it into a product space (Step~\ref{step.decompose}). We next describe two such constructions for 2D-translations and provide a construction for the 2D-affine group in the supplementary material \cite{OATM2018webpage}.

%-------------------------------------------------------------------------
\subsection{2D-translation constructions}\label{sec:construction_translation}
Recall that at the basis of our algorithm is the decomposition of the transformation search space $F$ into a product of spaces $F_{\epsilon'} \times Net_\epsilon(F)$, controlled by a parameter $\epsilon$. Depending on the structure of the space $F$ ($|F|=N$), we will pick a value of $\epsilon$ (and $\epsilon'$) for which $|F_{\epsilon'}| \approx |Net_\epsilon(F)| \approx \sqrt{N}$, in order to minimize the complexity which depends on the sum of the sizes of the product spaces. We make the decomposition explicit for the case of 2D-translations.

Since no scale is involved, $s(f^*)\hspace{-4pt}=\hspace{-4pt}1$ and hence $\epsilon'\hspace{-4pt}=\hspace{-4pt}\epsilon$. Given a square 
template $T$ and image $I$ of dimensions $m\times m$ and $n\times n$, the subspaces $F_{\epsilon}$ and $Net_\epsilon(F)$ can be constructed using a hexagonal cover of a square by circles of radius $\epsilon$, as is depicted in Figure~\ref{fig.illust.decomp}(a). The sizes of the resulting subspaces $F_{\epsilon}$ and $Net_\epsilon(F)$: $\pi \epsilon^2$ and $(n-m+1)^2/(1.5\sqrt{3}\epsilon^2)$, can be made equal by tuning $\epsilon$.

However, this covering is sub-optimal by a multiplicative factor of $1.5\sqrt{3}$ due to the overlap of circles.
We can actually get a practically optimal decomposition (while not strictly following the $\epsilon$-net definition), as is depicted in Figure~\ref{fig.illust.decomp}(b). We take the product of the sets: $F_{\epsilon}=\{i,j\: : \: i,j\in -\epsilon,\ldots,\epsilon\}$ and $Net_\epsilon(F)=\big\{i,j\: : \: i,j\in \{\epsilon+2k\epsilon\} \; \text{for} \; k = 1,\ldots,\lfloor(n-m+1)/2\epsilon)\rfloor\big\}$. This results in $|F_{\epsilon}|= 4\epsilon^2$ and $|Net_\epsilon(F)|= (n-m+1)^2/(4\epsilon^2)$. Taking $\epsilon=0.5\sqrt{n-m+1}$ yields $|F_{\epsilon}|=|Net_\epsilon(F)|=n-m+1$. % , as required.

%========================================================================
\section{Empirical validation of the analysis}
%========================================================================

% % % % % Figure Empirical Validation
\begin{figure}[t]\vspace{-9pt}
%\begin{center}
\hspace{-16pt}

\includegraphics[width = 1\linewidth]{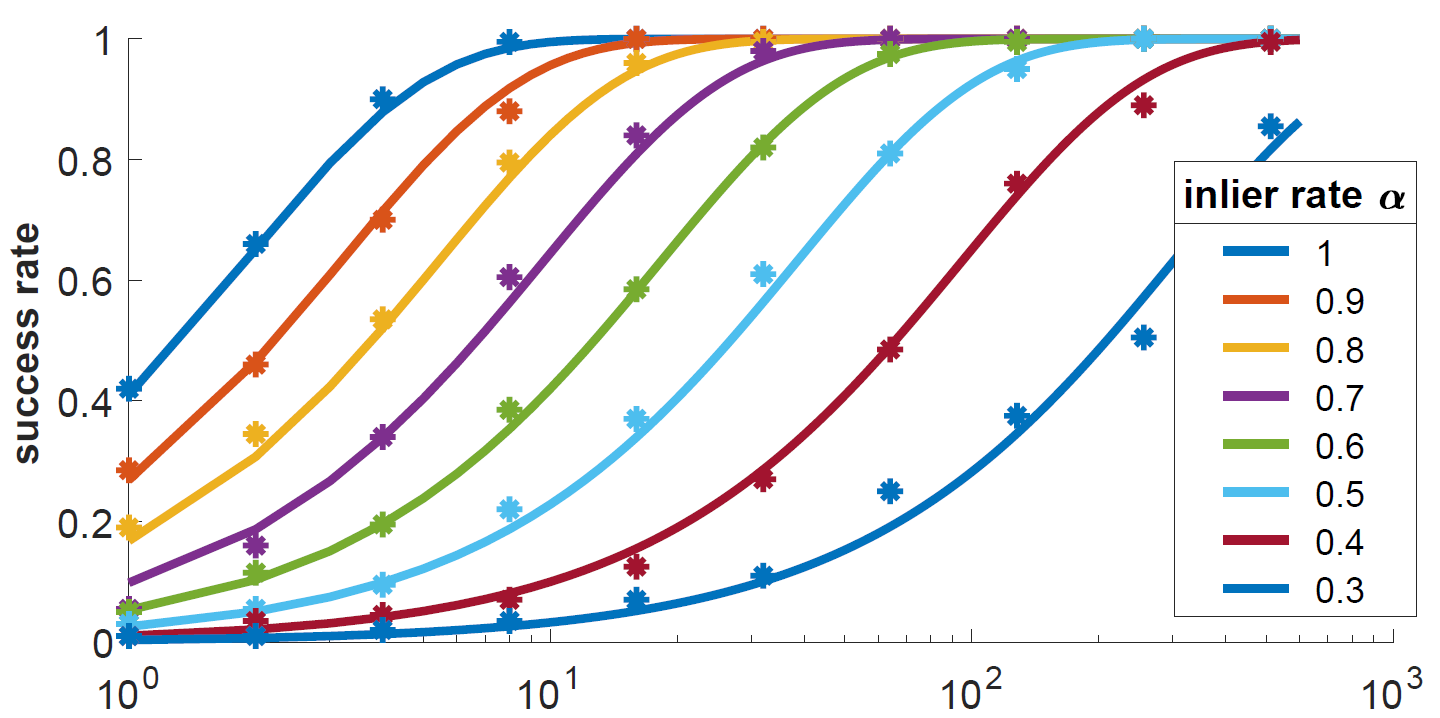}
\caption{\textbf{Empirical validation of Algorithm~\ref{alg:oatm}'s guarantees}. The theoretical success probabilities of OATM as a function of the number of iterations $k$ (solid curves) for different inlier rates $\alpha$ (notice the log-scale x-axis) can be seen to match the algorithm success rates (markers) measured in a large-scale experiment.}\label{fig.simulation}
%\end{center}
\vspace{-6pt}
\end{figure}
% % % % % Figure Empirical Validation

\vspace{3pt}
\noindent\textbf{Algorithm success rate (2D-translation)}
\vspace{2pt}

\noindent We begin with a large-scale validation of the theoretical guarantees of the algorithm (shown for the 2D-translation case), with each of the number $k$ of iterations in the set $\{1,2,4,8,16,32,64,128,256,512\}$, while the other parameters are kept fixed.

We run $200$ template matching trials for each inlier rate $\alpha$ in the set $\{0.3,0.4,0.5,0.6,0.7,0.8,0.9,1\}$. The success rate reported is the relative number of trials for which an exact match was found. For each trial we created a template matching instance, by first extracting a $100 \times 100$ template $T$ from a $500 \times 500$ image $I$ with grayscale intensities in [0,1], taken (scaled) at random from the Unplash data-set\footnote{A set of 65 high-res images we collected from https://unsplash.com/, which we present in the supplementary material \cite{OATM2018webpage}.}. A random $\alpha$-fraction of the template pixels are labeled as inlier pixels, and the intensity $T(p)$ of each outlier pixel $p$ is replaced with the intensity that is $0.5$ away from it in absolute difference. This setting guarantees that the resulting inlier rate is exactly $\alpha$, and the algorithm succeeds only if it samples a pure set of inliers. Finally, we add to the image $I$ white Gaussian noise with std equivalent of $5$ greylevels.

The results are shown in Figure~\ref{fig.simulation}, where the empirical success rates per $\alpha$ (markers) can be seen to match the  theoretical success rates from Equation~\eqref{eq.success.probab.alg2} (solid curves).
It is important to mention that these are minimal success rates guaranteed for finding the perfect match, which strictly hold, irrespective of the template and image contents, while in practice we often observe significantly better rates.

% % % % % Figure Scalability
\begin{figure}[t]\vspace{-9pt}
\hspace{-12pt}\includegraphics[width = 1.09\linewidth]{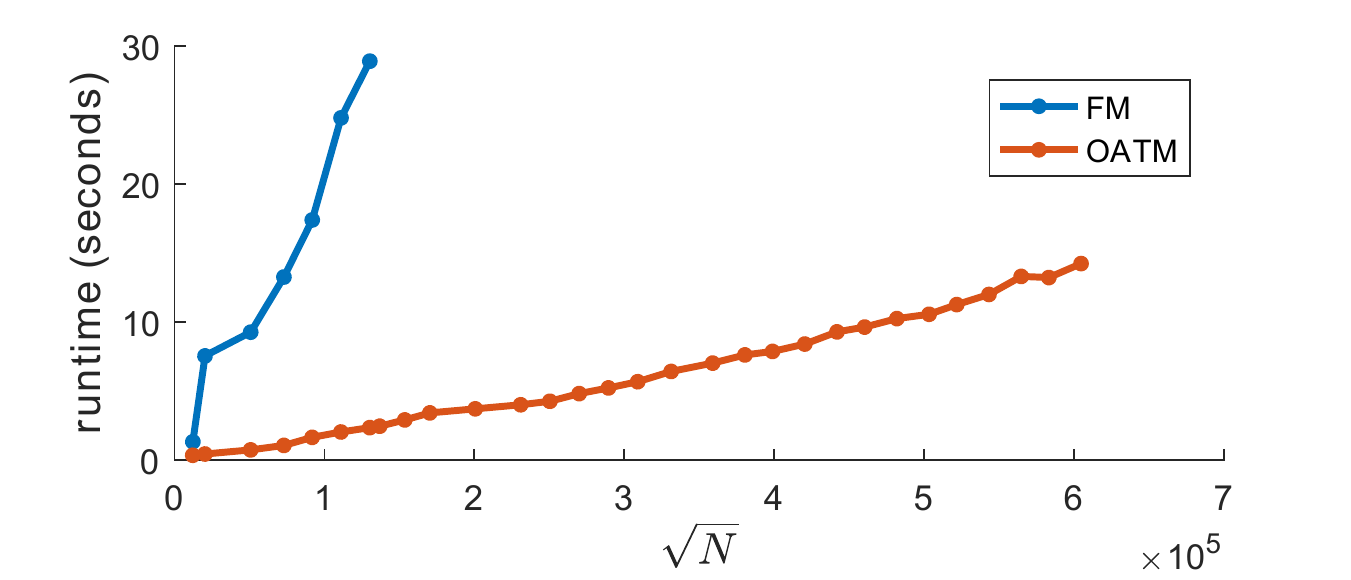} \vspace{-9pt}
\caption{\textbf{Scalability experiment}. OATM is compared empirically to FM, over a 2D-affine search space of size $N$. As expected, the runtime of OATM grows linearly in $\sqrt{N}$, while that of FM is linear in $N$ (notice the $\sqrt{N}$ $x$-axis).}\label{fig.scalability.plots}
%\end{center}
\vspace{-6pt}
\end{figure}
% % % % % Figure Scalability

\vspace{8pt}
\noindent\textbf{Algorithm scalability (2D-affine)}
\vspace{2pt}

\noindent In this experiment (result shown in Figure~\ref{fig.scalability.plots}) we verify the argued $O(\sqrt{N})$ runtime of our algorithm. A simple way of doing so is by creating a sequence of affine matching instances (see the experiment in Section \ref{sec:2da.synth.performance} for the technical details), were square templates of a fixed side length of $32$ pixels are searched in square images with varying side lengths in the set $\{100,200,300,\ldots,3200\}$, while keeping other affine search limits fixed - scales in the range $[2/3,3/2]$ and rotations in the range [$-\pi/4,\pi/4$]. This leads to a sequence of configuration sizes $N$ that grows quadratically (hence the markers are distributed roughly linearly in the $\sqrt{N}$ $x$-axis).
As can be seen, the runtime of OATM grows linearly with $\sqrt{N}$, and can handle in reasonable time a ratio of up to 100 between template and image dimensions. For reference, the complexity of the Fast-Match (FM) algorithm \cite{korman2013fast}, representing the state-of-the-art in affine template matching, depends on a standard parameterization of the 2D-affine space (whose size grows linearly in $N$ - see \cite{korman2013fast} for details). As can be seen, it cannot cope with ש template-image side length ratio of over 20.

\vspace{-5pt}
%========================================================================
\section{Results}\label{sec:results}
%========================================================================
\vspace{-3pt}

% % % % % Figure Inlier rates (occlusion experiment)
\begin{figure*}[t!]\vspace{-9pt}
\begin{center}
\addtolength{\tabcolsep}{-14.5pt}
\begin{tabular}{c  c}
    \hspace{9pt}\includegraphics[width = 0.535\textwidth]{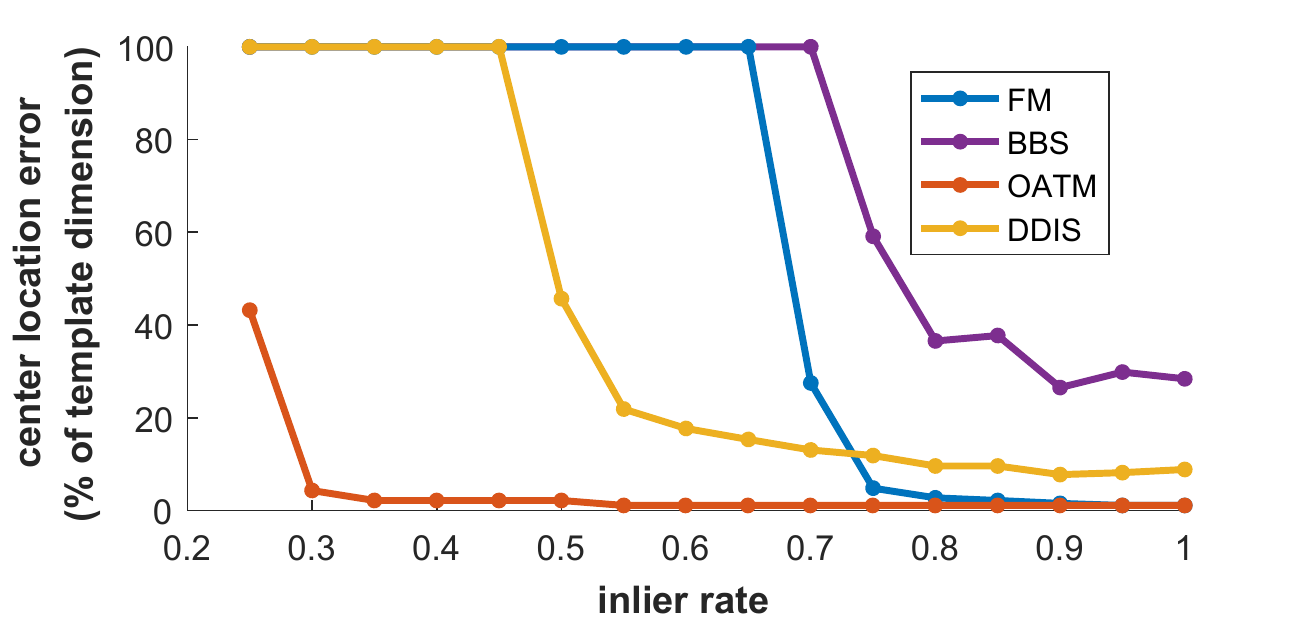} &
    \includegraphics[width = 0.535\linewidth]{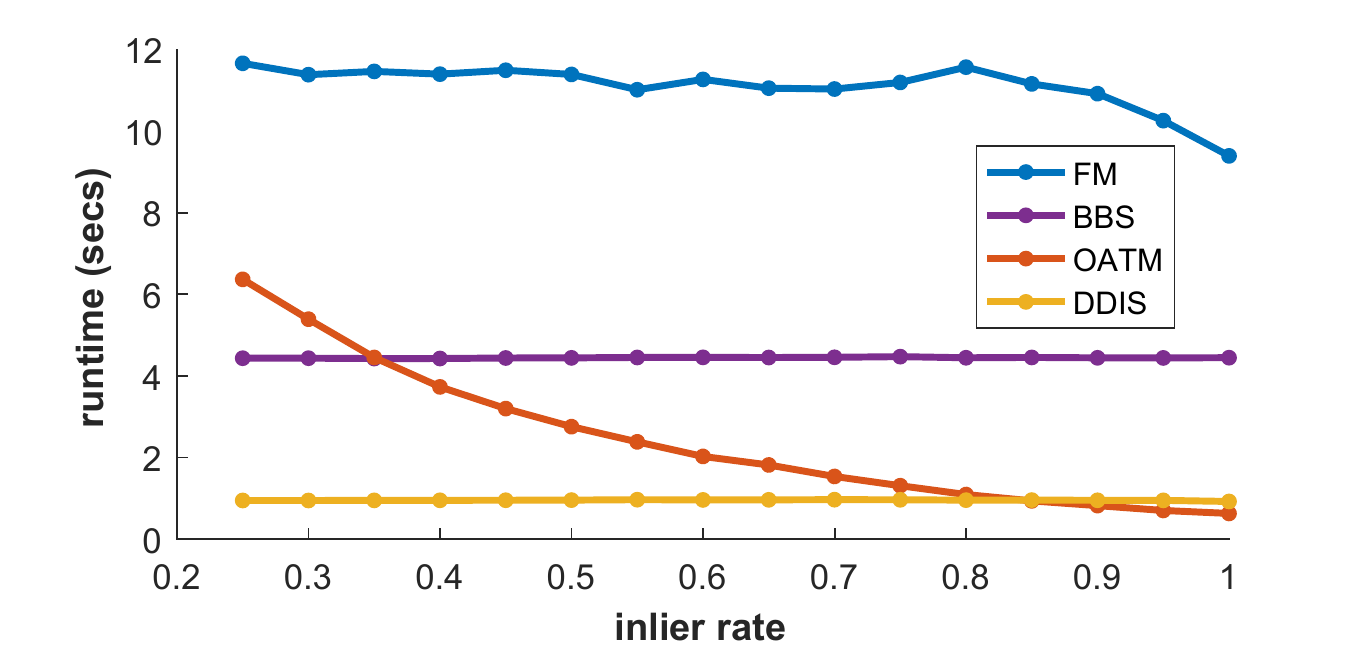} \\
      \vspace{-12pt}
\end{tabular}
\caption{\textbf{Results of the occlusion experiment (Sec. \ref{sec:occlusion.experiment}):} median center location errors (\textbf{left}) and average run-times (\textbf{right)}.}\label{fig.occlusion.plots}
\end{center}\vspace{-15pt}
\end{figure*}
% % % % % Figure Inlier rates (occlusion experiment)

In this section we demonstrate the advantages of the proposed algorithm through several controlled and uncontrolled experiments on real data.

\vspace{-6pt}
%-------------------------------------------------------------------------
\paragraph{Implementation details}
The parameters used in our implementation were chosen by simple coordinate descent over a small set of random synthetic instances (generated as described in Sec. \ref{sec:2da.synth.performance}). For the random grid, we use sample dimension $\hat{d} = 9$; cell dimension $c=2.5t$; where we take the threshold $t=2 \sigma \sqrt{2/\pi}$ (twice the mean of a zero-mean folded-normal-distribution), given a noise level of $\sigma$, or $t=10$ greylevels when it is unknown.
Our method can provide affine photometric invariance, i.e., global brightness and contrast changes, by standardizing the vector sets $U$ and $V$ (in step 2 of Algorithm \ref{alg:oatm}) to have the mean and standard deviation of the template.
%

%-------------------------------------------------------------------------
\subsection{Template matching evaluation}\label{sec:2da.synth.performance}

We test our algorithm in a standard template matching evaluation, not involving occlusions, in order to compare to other algorithms, such as Fast-Match (FM) \cite{korman2013fast} representing state-of-the-art in affine template matching. We run a large-scale comparison, using different combinations of template and image sizes (a larger gap between their sizes implies a larger size $N$ of the search space). We will use the following shorthands for template and image dimensions: T1 for $16\times 16$, T2 for $32\times 32$ and T3 for $64\times 64$. Likewise: I1 for $160\times 160$, I2 for $320\times 320$ and I3 for $640\times 640$.

For each template-image size combination, we ran 100 random template matching trials. Each trial (following~\cite{korman2013fast}) involves selecting a random image (here, from the Unplash data-set) and a random
affine transformation (parallelogram in the image).
The template is created by inverse-warping the parallelogram and white gaussian noise with $5$ graylevels equivalent std is added to the image.

For each trial we report average overlap errors and runtimes. The overlap error is a scalar in $[0,1]$ given by 1 minus the ratio between the intersection and union of the detected and true target parallelograms.

The results are summarized in Table \ref{tbl:template-image.dims}. OATM is typically an order of magnitude faster than FM, at similar low error levels. FM cannot deal with the setting T1-I3, due to the large number of configurations $N$ (the image edge length is 40 times the template edge length), while OATM deals with a more tolerable size of $\sqrt{N}$.

\begin{table}[h!] \vspace{-1pt}
  \centering
  \addtolength{\tabcolsep}{-5.7pt}{\footnotesize
  \begin{tabular}{|c|c|ccccccccc|}
    \cline{3-11}
    \multicolumn{1}{r}{}&&&\multicolumn{7}{c}{\footnotesize{\textbf{template-image sizes}}}&\\
    \hline
    \footnotesize{}&
    &\footnotesize{ T1-I1 }&\footnotesize{ T1-I2 }&\footnotesize{ T1-I3 }&\footnotesize{ T2-I1 } &\footnotesize{ T2-I2 }&\footnotesize{ T2-I3 }&\footnotesize{ T3-I1 }&\footnotesize{ T3-I2 }&\footnotesize{ T3-I3 } \\
    \hline
    \hline
%    \multirow{3}{*}{\footnotesize{FM}}
    \multirow{2}{*}{\footnotesize{FM}} % ~\cite{korman2013fast}
& \footnotesize{{err.}} & \footnotesize{0.09} & \footnotesize{0.13} & \footnotesize{NA} & \footnotesize{0.05} & \footnotesize{0.05} & \footnotesize{0.09} & \footnotesize{0.02} & \footnotesize{0.01} & \footnotesize{0.03} \\
%& \footnotesize{{success rate }} & -- & -- & -- & -- & -- & -- & -- & -- & -- \\
& \footnotesize{{time}} & \footnotesize{12.22} & \footnotesize{25.37} & \footnotesize{NA} & \footnotesize{4.35} & \footnotesize{7.78} & \footnotesize{32.07} & \footnotesize{1.33} & \footnotesize{1.90} & \footnotesize{11.61} \\
    \hline
%    \multirow{3}{*}{\footnotesize{OTM}}
    \multirow{2}{*}{\:\footnotesize{OATM}\:}
& \footnotesize{{err.}} & \footnotesize{0.07} & \footnotesize{0.10} & \footnotesize{0.13} & \footnotesize{0.02} & \footnotesize{0.04} & \footnotesize{0.04} & \footnotesize{0.01} & \footnotesize{0.02} & \footnotesize{0.13} \\
%& \footnotesize{{success rate }} & -- & -- & -- & -- & -- & -- & -- & -- & -- \\
& \footnotesize{\:{time}\:} & \footnotesize{0.15} & \footnotesize{0.18} & \footnotesize{0.39} & \footnotesize{0.53} & \footnotesize{0.76} & \footnotesize{1.73} & \footnotesize{0.51} & \footnotesize{0.64} & \footnotesize{1.01} \\
    \hline
  \end{tabular}}\vspace{5.5pt}
  \caption{\textbf{Template matching evaluation} for different template image sizes, including average runtime (seconds) and overlap error.} \vspace{-8pt}
  \label{tbl:template-image.dims}
\end{table}

%-------------------------------------------------------------------------
\subsection{Robustness to occlusions}\label{sec:occlusion.experiment}

In this experiment, we evaluate how well OATM and several other methods deal with occlusion. We repeat the protocol from the previous experiment (Section~\ref{sec:2da.synth.performance}), except that we take a fixed template-image size (T2-I2) and we synthetically introduce a controlled amount of outlier pixels. One way of doing so (see examples in Figure \ref{fig.occlusion.types}) is by introducing random $4 \times 4$ blocks. We repeated the experiment with two other ways of introducing occlusion, resulting in similar results, which we provide in the supplementary material \cite{OATM2018webpage}. These come to show that our method is robust to the spatial arrangement of the occlusion mask.

In addition to Fast-Match (FM) \cite{korman2013fast}, we compare with two additional template matching methods - Best Buddies Similarity (BBS) \cite{dekel2015best} and Deformable Diversity Similarity (DDIS) \cite{talmi2016template}, both specialized in handling complex geometric deformations and high levels of occlusion. For a fair comparison, since BBS and DDIS match the\ template in a sliding window fashion (and account for deformation within the window), we measure center location errors (rather than overlap error) - the distance between the center of the target window and the true target center location, as a percentage of the template dimension (clipped at 100\%).

The plots in Figure~\ref{fig.occlusion.plots} summarize the experiment. OATM can be seen to provide the most accurate detections at a very wide range of inlier rates, starting\ from around 0.25. DDIS can handle inlier rates of above 0.5, but is slightly less accurate in localization due to its sliding window search. FM was not designed to handle occlusions explicitly and fails to do so for inlier rates under 0.75. BBS does not handle inlier rates under 0.75 and its localization is suboptimal when dealing with the affine deformations in this setting.

In terms of speed, DDIS is clearly the most efficient. DDIS and BBS are agnostic of the inlier rate, while the runtime of OATM is inverse proportional to the inlier rate, due to its RANSAC-like adaptive stopping criterion.

%========================================================================
%-------------------------------------------------------------------------
\subsection{Matching partially occluded deformed patches} \label{sec:results.Hpatches}
%========================================================================

In this experiment we use the recent HPatches~\cite{hpatches_2017_cvpr} data-set, which was designed for benchmarking modern local image descriptors.
 %in the tasks of matching, retrieval and classification.
The patches were extracted from 116 sequences (59 with changing viewpoint, 57 with changing illumination), each containing 6 images of a planar scene with known geometric correspondence given by a 2D homography. Approximately 1300 square 65 $\times$ 65 reference patches (rectified state-of-the-art affine detected regions) are extracted from the first image in each sequence. The exact set of corresponding patches were then extracted from the 5 other sequence images, using the ground-truth projection, while introducing 3 levels (Easy, Hard, Tough) of controlled geometric perturbation (rotation, anisotropic scaling and translation), to simulate the location inaccuracies of current feature detectors.

These perturbations introduce significant geometric deformations (e.g. rotation of up to $10^\circ/20^\circ/30^\circ$) as well as increasing levels of occlusion (average overlap of $78\%/63\%/51\%$) for the Easy/Hard/Tough cases.
Figure \ref{fig.hpatches.data} shows several examples of extracted reference patches and their matching patches at the different levels of difficulty.

% % % % % Figure patch samples from HPatches dataset
\begin{figure}[h!]\vspace{-3pt}
\begin{center}
    \includegraphics[width = 1\linewidth]{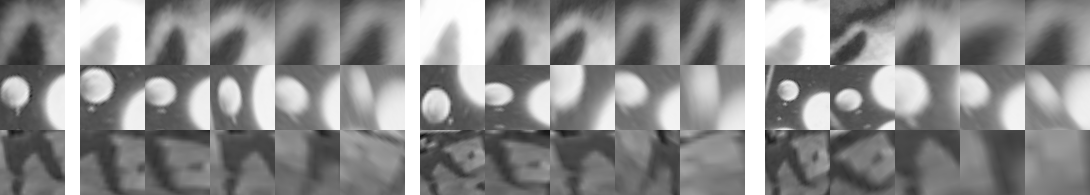} \\  \vspace{2pt}
    \includegraphics[width = 1\linewidth]{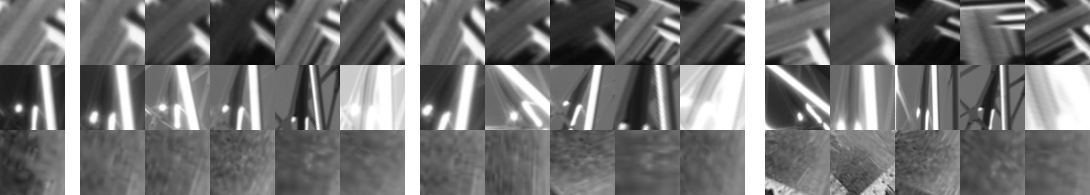} \\  \vspace{2pt}
    $\hspace{-2.3pt}$ ref $\hspace{3pt}$ E1 $\hspace{-1.8pt}$ E2 $\hspace{-1.8pt}$ E3 $\hspace{-1.8pt}$ E4 $\hspace{-1.8pt}$ E5 $\;$ H1 $\hspace{-2.1pt}$ H2 $\hspace{-2.1pt}$ H3 $\hspace{-2.1pt}$ H4 $\hspace{-2.1pt}$ H5 $\;$ T1 $\hspace{-2pt}$ T2 $\hspace{-2pt}$ T3 $\hspace{-2pt}$ T4 $\hspace{-2pt}$ T5
\caption{\textbf{Samples from the HPatches~\cite{hpatches_2017_cvpr} dataset.} \emph{viewpoint} sequences (rows 1-3) and \emph{illumination} sequences (rows 4-6).}\label{fig.hpatches.data}
\end{center}\vspace{-12pt}
\end{figure}
% % % % % Figure patch samples from HPatches dataset

This data is useful in showing the capabilities of our method in handling such challenges, in comparison with the common practice of matching features by their descriptors.
We focus on the proposed `matching' task \cite{hpatches_2017_cvpr}, in which each reference patch needs to be located among each of the patches of each sequence image. A template matching algorithm cannot strictly follow the suggested task protocol, which was defined for matching patches by their descriptors. Instead, we pack all the ($\sim$1300) square target patches into a single image in which we search for the template using the photometric invariant version of OATM. The target patch chosen is the one which contains the center location of the warped template patch. For mean-Average-Precision (mAP) calculation, since our method only produces a single target patch we assign a weight of 1 to the detected target patch and 0 to the rest. 

The results are summarized in Table \ref{tbl:Hpatches_results}. The reference descriptor methods include {SIFT}~\cite{lowe1999object} and its variant {RSIFT}~\cite{arandjelovic2012three}, the binary descriptors {BRIEF}~\cite{calonder2010brief} and {ORB}~\cite{rublee2011orb} and the deep descriptors DeepDesc ({DDESC})~\cite{simo2015discriminative} and TFeat ratio* ({TF-R})~\cite{balntas2016learning}. For SIFT, TF-R, DDESC and RSIFT, results are given for the superior whitened and normalized versions of the descriptors (as reported in ~\cite{hpatches_2017_cvpr}).

% % % % % HPatches benchmark result - summarized in a table

\begin{table}[h!] \vspace{-5pt}
  \centering
  \addtolength{\tabcolsep}{-2.2pt}{\small
  \begin{tabular}{|c|c|c|c|c|c|c|}
    \cline{1-7}
    &\multicolumn{3}{|c|}{\small{\textbf{viewpoint seqs}}}
    &\multicolumn{3}{|c|}{\small{\textbf{illumination seqs}}}\\
    \hline
    \small{\textbf{method}}&\small{ Easy }&\small{ Hard }&\small{ Tough }&\small{ Easy }&\small{ Hard }&\small{ Tough } \\
    \hline
\hline
{\small{BRIEF} \cite{calonder2010brief}} & {}\small{25.6 } & {}\small{6.9 } & {}\small{2.4 }  &
{}\small{20.5 } & {}\small{5.9 } & {}\small{2.0 }   \\
 %\cline{1-1}
{\small{ORB} \cite{rublee2011orb}} & {}\small{36.4} & {}\small{11.1 } & {}\small{3.7 }  &
{}\small{28.9 } & {}\small{8.8 } & {}\small{3.2 }  \\
 %\cline{1-1}
{\small{SIFT} \cite{lowe1999object}} & {}\small{59.4 } & {}\small{30.6 } & {}\small{15.3 } &
{}\small{52.6 } & {}\small{26.1 } & {}\small{13.3 }   \\
 %\cline{1-1}
{\small{TF-R} \cite{balntas2016learning}} & {}\small{58.9 } & {}\small{35.5 } & {}\small{19.0 } &
{}\small{48.5 } & {}\small{28.6 } & {}\small{15.6 }   \\
 %\cline{1-1}
{\small{DDESC} \cite{simo2015discriminative}} & {}\small{58.6 } & {}\small{36.0 } & {}\small{20.2 }  &
{}\small{50.7 } & {}\small{30.0} & {}\small{17.0 }   \\
 %\cline{1-1}
{\small{RSIFT} \cite{arandjelovic2012three}} & {}\small{64.0} & {}\small{35.2} & {}\small{18.5}  &
{}\small{\textbf{57.1} } & {}\small{\textbf{30.2 }} & {}\small{15.9 }  \\
 %\cline{1-1}
 %\hline
{\small{OATM}} & {}\small{\textbf{72.7} } & {}\small{\textbf{49.2}} & {}\small{\textbf{32.1}} &
{}\small{43.3 } & {}\small{29.3} & {}\small{\textbf{19.7}}  \\
 %\cline{1-1}
\hline
%
%\hline
%{\textbf{\# of instances}} & {}\small{76} & {}\small{96} & {}\small{138}   \\
%\cline{1-1}
%\hline
%
  \end{tabular}}\vspace{4pt}
  \caption{\textbf{Results on the HPatches \cite{hpatches_2017_cvpr} Image Matching benchmark.} Results are in terms of mean-Average-Precision (mAP), where all results except that of OATM were reported in \cite{hpatches_2017_cvpr}.} \vspace{-7pt}
  \label{tbl:Hpatches_results}
\end{table} 
% % % % % HPatches benchmark result - summarized in a table

Clearly, for both viewpoint and illumination sequences - the mAP of OATM deteriorates more gracefully with the increase in geometric deformation and level of occlusion, compared to the descriptor based methods. While the state-of-the-art features and descriptors may be highly insensitive to certain local geometric deformations and different photometric variations (and hence some outperform OATM in the Easy illumination case), they are not as effective in dealing with significant deformation and occlusion, unlike OATM which explicitly explores the space of affine deformations and reasons about substantial occlusion levels.

Furthermore, the naive current application of OATM on this data suggests that performance could be further improved by: (i) finding a distribution over target locations rather than one single detection; (ii) being aware of the patch structure of the stacked target image; (iii) using advanced representations instead of the greylevel pixelwise description.
That being said, unlike the descriptor based methods, the template matching nature of OATM is certainly not suitable for large-scale matching, where a large pool of patches needs to be matched against another. Nevertheless, many of the ideas presented here could be possibly adapted, e.g. to the image-to-image matching setup.

\vspace{-3pt}
%========================================================================
\section{Conclusions}
%========================================================================
\vspace{-2pt}

We have presented a highly efficient algorithm for 2D-affine template matching that is carefully analyzed and is shown to improve on previous methods in handling high levels of occlusion and geometric deformation.

The results on the HPatches data-set raise the question of whether descriptor based matching is able to handle the geometric deformations and high occlusion levels that are inherent in the localization noise introduced by feature detectors. This is the case even in the advent of deep learning, and the development of methods that can explicitly reason for deformation and occlusion seems to be necessary for improving the state-of-the-art in visual correspondence.

\vspace{-8pt}
%========================================================================
\paragraph{Acknowledgement}Research supported by ARO W91 1NF-15-1-0564/66731-CS, ONR N00014-17-1-2072 and a gift from Northrop Grumman.
%========================================================================
\vspace{-2pt}

{\small
\bibliographystyle{ieee}
\bibliography{egbib}
}

\end{document}